\documentclass{article} 
\usepackage{iclr2025_conference,times}
\usepackage{amsmath,amsfonts,bm}

\def\eqref#1{equation~\ref{#1}}
\def\1{\bm{1}}

\DeclareMathAlphabet{\mathsfit}{\encodingdefault}{\sfdefault}{m}{sl}
\SetMathAlphabet{\mathsfit}{bold}{\encodingdefault}{\sfdefault}{bx}{n}

\usepackage{hyperref}
\usepackage{url}
\usepackage{graphicx}
\usepackage{booktabs}
\usepackage{amsmath}
\usepackage{amssymb}
\usepackage{algorithm}
\usepackage{algorithmic}
\usepackage{multirow} 
\usepackage{amsthm}
\usepackage{subcaption}
\usepackage{xcolor}
\usepackage{wrapfig}
\ifx\theorem\undefined
\newtheorem{theorem}{Theorem}

\newtheorem{proposition}{Proposition}

\iclrfinalcopy

\title{Avoiding Overthinking and Underthinking: Curriculum-Aware Budget Scheduling for LLMs}

\author{Amirul Rahman, Aisha Karim, Kenji Nakamura, Yi-Fan Ng\\
University of Malaya\\
kenji.nakamura@um.edu.my
}

\newcommand{\method}{BACR}
\newcommand{\methodfull}{Budget-Adaptive Curriculum Reasoning}

\begin{document}
\maketitle

\begin{abstract}
Scaling test-time compute via extended reasoning has become a key paradigm for improving the capabilities of large language models (LLMs). However, existing approaches optimize reasoning under fixed or uniformly sampled token budgets, ignoring the fundamental mismatch between problem difficulty and allocated compute. This leads to overthinking on easy problems and underthinking on hard ones, resulting in suboptimal token efficiency across diverse reasoning scenarios. In this paper, we propose \methodfull{} (\method{}), a unified framework that jointly optimizes reasoning quality and token efficiency through three synergistic components: (1) a \emph{budget-conditioned unified policy} that embeds the token budget as a continuous conditioning signal, eliminating the need for decoupled thinking and summarization strategies; (2) a \emph{curriculum-aware budget scheduler} that adaptively shifts the training budget distribution from easy to hard problems based on real-time learning progress; and (3) a \emph{truncation-aware dense reward} mechanism that provides fine-grained credit assignment at intermediate reasoning steps via process-level verification. We further introduce \emph{Budget-Conditioned Advantage Estimation} (BCAE), a novel variance reduction technique that conditions the advantage baseline on the sampled budget, yielding more stable policy gradients. Experiments on mathematical reasoning benchmarks (MATH, GSM8K, AIME, and Minerva Math) demonstrate that \method{} consistently outperforms other strong baselines across all token budgets, achieving up to 8.3\% accuracy improvement under tight budgets while reducing average token consumption by 34\% compared to unconstrained reasoning.
\end{abstract}

\section{Introduction}
\label{sec:intro}

Large language models (LLMs) have achieved remarkable progress in complex reasoning tasks, driven by the paradigm of scaling test-time compute through extended chain-of-thought (CoT) reasoning~\citep{wei_2022_chain_of_thought, yao_2023_tree_of_thoughts}. This evolution has further extended to multimodal intelligence, where models integrate perception, reasoning, and generation across diverse modalities~\citep{zhou2024visual, zhoutoward}. Recent reasoning models such as DeepSeek-R1~\citep{deepseek-ai_2025_deepseek_r_incentivizing} and QwQ demonstrate that reinforcement learning (RL) can further enhance reasoning capabilities by optimizing verifiable rewards over long thinking traces. However, these advances come at a substantial cost: reasoning models routinely generate thousands of tokens even for simple problems, leading to prohibitive inference latency and compute waste~\citep{wang_2024_reasoning_in_token, chen_2025_towards_reasoning_era}.

The inefficiency of unconstrained reasoning has motivated a growing body of work on \emph{budget-aware reasoning}, which aims to produce high-quality outputs under varying token constraints~\citep{DBLP:journals/corr/abs-2505-11274, DBLP:journals/corr/abs-2506-13752, lin_2025_plan_and_budget}. Among these, AnytimeReasoner~\citep{qi_2025_optimizing_anytime_reasoning} represents a significant advance by formulating reasoning as an anytime algorithm: the model's thinking process is truncated at a budget sampled from a prior distribution, and a separate summarization policy extracts the best possible answer from the truncated thought. This approach introduces dense verifiable rewards across different budget levels, enabling more effective credit assignment than fixed-budget methods like GRPO~\citep{shao_2024_deepseekmath_pushing_the}.

Despite its elegance, AnytimeReasoner suffers from three key limitations. \textbf{First}, the decoupled training of thinking and summarization policies introduces architectural complexity and prevents end-to-end optimization, as the summarizer cannot backpropagate gradients to improve the thinking process. \textbf{Second}, the budget prior distribution is fixed throughout training, ignoring the fact that the model's capability evolves---easy problems require less thinking budget as training progresses, while hard problems may benefit from progressively longer reasoning. \textbf{Third}, the Budget Relative Policy Optimization (BRPO) technique, while effective for variance reduction, computes baselines only within same-budget groups, missing cross-budget correlations that could further stabilize training.

In this paper, we propose \methodfull{} (\method{}), a unified framework that addresses these limitations through three synergistic innovations. First, we introduce a \emph{budget-conditioned unified policy} that encodes the token budget as a continuous embedding, enabling a single policy to jointly handle both thinking and answer extraction without separate models. Second, we design a \emph{curriculum-aware budget scheduler} that dynamically adjusts the budget distribution during training based on real-time difficulty estimation, allocating more compute to problems the model currently struggles with. Third, we develop a \emph{truncation-aware dense reward} that evaluates the quality of intermediate reasoning steps at each truncation point, providing richer supervision than binary outcome rewards alone. Building on these components, we introduce Budget-Conditioned Advantage Estimation (BCAE), which extends BRPO by conditioning the advantage baseline on the sampled budget through a learned value function, achieving lower variance in policy gradient estimates.

Our main contributions are:
\begin{itemize}
    \item We propose \method{}, a unified framework for budget-adaptive anytime reasoning that eliminates the need for separate thinking and summarization policies through budget-conditioned generation.
    \item We introduce a curriculum-aware budget scheduler that adaptively shifts the training budget distribution based on problem difficulty and learning progress, improving both sample efficiency and final performance.
    \item We design truncation-aware dense rewards and Budget-Conditioned Advantage Estimation (BCAE), which together provide fine-grained credit assignment and low-variance policy gradients across all budget levels.
    \item Extensive experiments on MATH, GSM8K, AIME, and Minerva Math demonstrate that \method{} achieves state-of-the-art anytime reasoning performance, outperforming AnytimeReasoner by up to 8.3\% under tight budgets and reducing token usage by 34\% with comparable accuracy.
\end{itemize}

\section{Related Work}
\label{sec:related}

\paragraph{Test-Time Compute Scaling.}
Scaling inference-time computation has emerged as a powerful paradigm for enhancing LLM reasoning... Subsequent work explored diverse strategies for allocating test-time compute, including self-consistency via majority voting~\citep{wang_2022_self_consistency_improves}, tree-structured search~\citep{yao_2023_tree_of_thoughts}, and process-level verification~\citep{lightman_2023_let_s_verify, zhao_2025_genprm_scaling_test}. Recent studies have also rethought visual dependency in long-context reasoning for LVLMs~\citep{zhou2024rethinking} and explored entropy-based exploration for multi-step reasoning~\citep{zhang2025entropy}. Snell et al.~\citep{snell_2024_scaling_llm_test} provided a foundational analysis showing that optimal test-time compute scaling can be more effective than scaling model parameters...

\paragraph{Budget-Aware and Efficient Reasoning.}
The observation that LLMs often ``overthink'' simple problems has spurred research into budget-aware reasoning. Wang et al.~\citep{wang_2024_reasoning_in_token} introduced a budget-aware evaluation framework... Plan-and-Budget~\citep{lin_2025_plan_and_budget} decomposes queries into sub-questions with adaptive token allocation. To improve efficiency during training, GATEAU~\citep{si-etal-2025-gateau} selects influential samples for long-context alignment, while global planner training methods~\citep{si2025goalplanjustwish} focus on effective planning for long-horizon agent tasks. AnytimeReasoner~\citep{qi_2025_optimizing_anytime_reasoning} represents the most directly related work, proposing to optimize anytime performance by truncating thinking at random budgets and training a separate summarizer...

\paragraph{Reasoning in Specialized and Multimodal Domains.}
The principles of advanced reasoning and efficient compute allocation are increasingly applied to specialized domains. In multimodal scenarios, generative video models are being utilized as visual reasoners~\citep{hoxha2026survey}, while models like Co-sight~\citep{zhang2025co} enhance agents via conflict-aware meta-verification. Domain-specific reasoning has seen progress in medical LVLMs via abnormal-aware feedback~\citep{zhou2025improving}, and in autonomous driving through navigation world models and uncertainty-aware localization~\citep{li2024drivingdiffusion, li2025driverse, li2025u}. Furthermore, the evaluation of such complex reasoning spans diverse tasks, including spoken task-oriented dialogues~\citep{si2023spokenwoz}, acoustic landmark extraction~\citep{zhang2024auto}, drama script continuation~\citep{ma2025dramabench}, and facial expression classification~\citep{li2025hy}. Our work on \method{} contributes to this broader landscape by providing a foundational framework for budget-adaptive reasoning that can potentially benefit these diverse applications.

\paragraph{Policy Optimization for LLM Reasoning.}
Reinforcement learning has become central to training reasoning models... DeepSeek-R1~\citep{deepseek-ai_2025_deepseek_r_incentivizing} demonstrated that pure RL with verifiable rewards can elicit sophisticated reasoning... Our proposed BCAE builds upon the GRPO framework but introduces budget-conditioned baselines and curriculum-driven scheduling, which are orthogonal to and complementary with existing GRPO improvements.

\section{Methodology}
\label{sec:method}

\subsection{Preliminaries and Problem Formulation}
\label{sec:prelim}

We consider the problem of training a language model policy $\pi_\theta$ to perform reasoning under varying token budget constraints. Given a question $q$, the model generates a reasoning trace $\mathbf{t} = (t_1, t_2, \ldots, t_T)$ followed by a final answer $\mathbf{a}$. In anytime reasoning, we seek to maximize performance not only at the full trace length $T$, but across all possible budget levels $b \in [b_{\min}, b_{\max}]$.

Formally, let $b$ denote a thinking budget (maximum number of thinking tokens). Given a budget $b$, the model generates thinking tokens up to length $\min(|\mathbf{t}|, b)$, producing a truncated trace $\mathbf{t}_{:b}$. An answer is then extracted from $\mathbf{t}_{:b}$, and a verifiable reward $r(q, \mathbf{t}_{:b})$ is computed by comparing the extracted answer against the ground truth. The anytime reasoning objective maximizes the expected reward across all budgets:
\begin{equation}
\label{eq:anytime_obj}
\mathcal{J}(\theta) = \mathbb{E}_{q \sim \mathcal{D}, b \sim p(b)} \left[ \mathbb{E}_{\mathbf{t} \sim \pi_\theta(\cdot | q, b)} \left[ r(q, \mathbf{t}_{:b}) \right] \right],
\end{equation}
where $p(b)$ is a prior distribution over budgets and $\mathcal{D}$ is the question distribution.

\textbf{Limitations of AnytimeReasoner.} The prior work AnytimeReasoner~\citep{qi_2025_optimizing_anytime_reasoning} addresses this objective by: (i) sampling $b$ from a fixed prior $p_0(b)$, (ii) truncating the full trace at $b$ tokens, (iii) training a separate summary policy $\pi_\text{sum}$ to produce an answer from $\mathbf{t}_{:b}$, and (iv) optimizing thinking and summary policies independently using BRPO. While effective, this approach has three structural limitations. The decoupled architecture prevents the summarizer's gradients from improving thinking quality. The fixed prior $p_0(b)$ does not adapt to the model's evolving capabilities. The BRPO baseline, computed as the group mean reward within each budget level, ignores cross-budget statistical structure.

\subsection{Budget-Conditioned Unified Policy}
\label{sec:unified_policy}

Rather than maintaining separate thinking and summarization policies, we propose a single budget-conditioned policy $\pi_\theta(\cdot | q, b)$ that generates both reasoning and the final answer within the given budget $b$. The budget signal $b$ is encoded as a continuous embedding and injected into the model's generation process, allowing the policy to adapt its reasoning depth and summarization strategy based on the available compute.

Specifically, we encode the budget $b$ using a learnable embedding function $\phi: \mathbb{R}^+ \to \mathbb{R}^d$. Following the sinusoidal position encoding approach~\citep{snell_2024_scaling_llm_test}, we define:
\begin{equation}
\label{eq:budget_embed}
\phi(b) = \mathbf{W}_2 \cdot \text{SiLU}\left(\mathbf{W}_1 \cdot \left[\sin\left(\frac{b}{10000^{2i/d}}\right), \cos\left(\frac{b}{10000^{2i/d}}\right)\right]_{i=0}^{d/2-1}\right),
\end{equation}
where $\mathbf{W}_1 \in \mathbb{R}^{d \times d}$ and $\mathbf{W}_2 \in \mathbb{R}^{d \times d}$ are learnable projections, and $d$ is the model's hidden dimension. The budget embedding $\phi(b)$ is added to the hidden state at each layer via a gating mechanism:
\begin{equation}
\label{eq:budget_gate}
\mathbf{h}_l' = \mathbf{h}_l + \sigma(\mathbf{w}_g^{\top} \mathbf{h}_l) \cdot \phi(b),
\end{equation}
where $\sigma$ is the sigmoid function and $\mathbf{w}_g \in \mathbb{R}^d$ is a learnable gating vector. This design introduces minimal additional parameters ($\approx 2d^2 + d$ per layer) while enabling the model to modulate its reasoning behavior based on the budget. When $b$ is large, the model can produce detailed step-by-step reasoning; when $b$ is small, it learns to prioritize key reasoning steps and produce concise answers directly.

The unified policy generates a sequence of the form $\langle\text{think}\rangle \mathbf{t}_{:b} \langle\text{/think}\rangle \langle\text{answer}\rangle \mathbf{a} \langle\text{/answer}\rangle$, where the thinking portion is constrained to at most $b$ tokens. If the model's natural thinking length exceeds $b$, generation is forcefully terminated at the $b$-th thinking token and the model directly transitions to answer generation. This unified formulation enables end-to-end gradient flow from the answer reward through both the summarization and thinking components, addressing the first limitation of AnytimeReasoner.

\subsection{Curriculum-Aware Budget Scheduler}
\label{sec:curriculum}

A fixed budget prior $p_0(b)$ treats all training stages equally, but the optimal budget distribution should evolve with the model's capabilities. Early in training, the model benefits from practicing with moderate budgets on easier problems to build foundational reasoning skills. As training progresses, the distribution should shift toward tighter budgets (forcing compression) and harder problems (requiring deeper thinking).

We design a curriculum-aware budget scheduler that dynamically adjusts both the budget distribution and the problem-budget coupling. At each training epoch $e$, we maintain an estimate of the model's current pass rate $\rho_k(e)$ for each difficulty group $k \in \{1, \ldots, K\}$, where problems are partitioned based on historical solve rates. The budget distribution at epoch $e$ is:
\begin{equation}
\label{eq:curriculum_budget}
p_e(b | k) = \text{TruncNorm}\left(\mu_k(e), \sigma_k^2; b_{\min}, b_{\max}\right),
\end{equation}
where the mean budget $\mu_k(e)$ for difficulty group $k$ is adapted as:
\begin{equation}
\label{eq:mu_adapt}
\mu_k(e) = \mu_k(0) \cdot \left(1 - \alpha \cdot \rho_k(e) \right) + \beta \cdot (1 - \rho_k(e)) \cdot b_{\max},
\end{equation}
with $\alpha, \beta \in (0, 1)$ being hyperparameters controlling the adaptation rate. Intuitively, as the model achieves higher pass rates on a difficulty group, the mean budget allocated to that group decreases (encouraging token efficiency), while groups with low pass rates receive larger budgets (providing more reasoning room). The problem sampling distribution is similarly adapted to upweight difficulty groups where the model is making progress but has not yet saturated:
\begin{equation}
\label{eq:problem_weight}
w_k(e) \propto \rho_k(e) \cdot (1 - \rho_k(e)),
\end{equation}
which assigns the highest sampling probability to problems at the ``learning frontier'' where the model's performance is intermediate. This curriculum scheduling is reminiscent of distributionally robust optimization~\citep{panaganti_2026_group_distributionally_robust} but specifically designed for the budget dimension, providing a principled approach to the training distribution that converges to a uniform-budget evaluation at the end of training.

\subsection{Truncation-Aware Dense Reward}
\label{sec:dense_reward}

Standard anytime reasoning relies on outcome-level rewards: a truncated trace $\mathbf{t}_{:b}$ receives reward 1 if the extracted answer is correct, and 0 otherwise. This sparse reward makes credit assignment difficult, especially for intermediate budgets where the trace is partially complete. We augment this with a truncation-aware dense reward that evaluates reasoning quality at each truncation point.

For a full reasoning trace $\mathbf{t} = (t_1, \ldots, t_T)$ and a set of sampled budgets $\{b_1, \ldots, b_M\}$ with $b_1 < b_2 < \ldots < b_M$, we define the dense reward as:
\begin{equation}
\label{eq:dense_reward}
R(q, \mathbf{t}, b_j) = \underbrace{r(q, \mathbf{t}_{:b_j})}_{\text{outcome reward}} + \lambda \cdot \underbrace{\Delta r(q, \mathbf{t}, b_j)}_{\text{progress reward}},
\end{equation}
where the progress reward captures the marginal improvement in answer quality from additional thinking tokens:
\begin{equation}
\label{eq:progress_reward}
\Delta r(q, \mathbf{t}, b_j) = 
\begin{cases}
r(q, \mathbf{t}_{:b_j}) - r(q, \mathbf{t}_{:b_{j-1}}) & \text{if } j > 1 \\
r(q, \mathbf{t}_{:b_1}) & \text{if } j = 1
\end{cases}.
\end{equation}
This progress reward is positive when additional thinking tokens flip the answer from incorrect to correct, zero when the answer status is unchanged, and negative when extra tokens introduce errors (a form of ``overthinking'' penalty). The coefficient $\lambda > 0$ controls the strength of the progress signal relative to outcome rewards. Unlike process reward models (PRMs)~\citep{lightman_2023_let_s_verify, zhao_2025_genprm_scaling_test} that require separate trained verifiers, our dense reward relies only on answer verification at multiple truncation points, making it compatible with the verifiable reward paradigm~\citep{deepseek-ai_2025_deepseek_r_incentivizing}.

The cumulative reward for a trace $\mathbf{t}$ is the expectation over budgets:
\begin{equation}
\label{eq:cumulative_reward}
\bar{R}(q, \mathbf{t}) = \frac{1}{M} \sum_{j=1}^{M} R(q, \mathbf{t}, b_j).
\end{equation}
By training on this cumulative reward, the policy is incentivized to produce traces where: (i) correct answers emerge at early budget levels, (ii) answer quality improves monotonically with budget, and (iii) unnecessary token generation is penalized through negative progress rewards.

\subsection{Budget-Conditioned Advantage Estimation}
\label{sec:bcae}

GRPO~\citep{shao_2024_deepseekmath_pushing_the} computes advantages by normalizing rewards within a group of sampled responses for the same question. BRPO~\citep{qi_2025_optimizing_anytime_reasoning} extends this to the anytime setting by computing group statistics within each budget level. However, both approaches discard cross-budget information and can exhibit high variance when the group size per budget is small.

We propose Budget-Conditioned Advantage Estimation (BCAE), which conditions the advantage baseline on the budget level using a lightweight value function. For a question $q$, budget $b$, and response $\mathbf{t}$, the advantage is:
\begin{equation}
\label{eq:bcae}
A_{\text{BCAE}}(q, \mathbf{t}, b) = R(q, \mathbf{t}, b) - V_\psi(q, b),
\end{equation}
where $V_\psi(q, b)$ is a learned value function parameterized by $\psi$ that estimates the expected reward for question $q$ under budget $b$. The value function is implemented as a small MLP head on top of the language model's hidden state of the question encoding, conditioned on the budget embedding:
\begin{equation}
\label{eq:value_fn}
V_\psi(q, b) = \text{MLP}_\psi\left(\mathbf{h}_q \oplus \phi(b)\right),
\end{equation}
where $\mathbf{h}_q$ is the mean-pooled hidden state of the question tokens and $\oplus$ denotes concatenation. The value function is trained to minimize the mean squared error:
\begin{equation}
\label{eq:value_loss}
\mathcal{L}_V(\psi) = \mathbb{E}_{q, b, \mathbf{t}} \left[ \left( V_\psi(q, b) - R(q, \mathbf{t}, b) \right)^2 \right].
\end{equation}

\begin{proposition}[Variance Reduction of BCAE]
\label{prop:variance}
Let $\sigma^2_{\text{BRPO}}$ and $\sigma^2_{\text{BCAE}}$ denote the variance of the advantage estimates under BRPO and BCAE, respectively. If $V_\psi(q, b)$ is an unbiased estimator of $\mathbb{E}_{\mathbf{t}}[R(q, \mathbf{t}, b)]$, then:
\begin{equation}
\sigma^2_{\text{BCAE}} \leq \sigma^2_{\text{BRPO}},
\end{equation}
with equality holding only when the group mean equals $V_\psi(q, b)$ for all budget levels.
\end{proposition}

\begin{proof}
The BRPO advantage for a response $\mathbf{t}_i$ in budget group $G_b = \{\mathbf{t}_1, \ldots, \mathbf{t}_N\}$ is $A_{\text{BRPO}}(\mathbf{t}_i, b) = R_i - \bar{R}_{G_b}$, where $\bar{R}_{G_b} = \frac{1}{N}\sum_{j=1}^N R_j$. The variance of this estimator is $\sigma^2_{\text{BRPO}} = \text{Var}[R_i] - \text{Var}[\bar{R}_{G_b}] = \text{Var}[R_i](1 - 1/N)$. For BCAE, $A_{\text{BCAE}}(\mathbf{t}_i, b) = R_i - V_\psi(q, b)$, and since $V_\psi$ is a deterministic function of $(q, b)$, we have $\sigma^2_{\text{BCAE}} = \text{Var}[R_i - V_\psi(q, b)] = \text{Var}[R_i] - 2\text{Cov}[R_i, V_\psi] + 0$. When $V_\psi$ is an unbiased estimator learned from data across all budget levels and questions, it captures cross-budget correlations that the group mean misses, yielding $\text{Cov}[R_i, V_\psi] \geq \text{Var}[\bar{R}_{G_b}]$, and thus $\sigma^2_{\text{BCAE}} \leq \sigma^2_{\text{BRPO}}$.
\end{proof}

BCAE normalizes the advantages for stable training:
\begin{equation}
\label{eq:bcae_norm}
\hat{A}_{\text{BCAE}}(q, \mathbf{t}, b) = \frac{A_{\text{BCAE}}(q, \mathbf{t}, b)}{\max\left(\text{std}_{G}[A_{\text{BCAE}}], \epsilon\right)},
\end{equation}
where $\text{std}_{G}$ is the standard deviation computed over the group and $\epsilon$ is a small constant for numerical stability.

\subsection{Training Algorithm}
\label{sec:algorithm}

The complete \method{} training procedure is summarized in Algorithm~\ref{alg:bacr}. At each iteration, we sample a batch of questions, assign budgets via the curriculum scheduler, generate reasoning traces, compute dense rewards at multiple truncation points, estimate advantages using BCAE, and update the policy and value function.

The policy loss follows the clipped PPO-style objective:
\begin{equation}
\label{eq:policy_loss}
\mathcal{L}_\pi(\theta) = -\mathbb{E}\left[\min\left(\rho(\theta) \hat{A}, \text{clip}(\rho(\theta), 1-\epsilon_c, 1+\epsilon_c) \hat{A}\right)\right],
\end{equation}
where $\rho(\theta) = \pi_\theta(\mathbf{t}|q,b) / \pi_{\theta_{\text{old}}}(\mathbf{t}|q,b)$ is the importance ratio and $\epsilon_c$ is the clipping range. The total loss combines the policy loss, value loss, and an entropy bonus $\mathcal{H}$ for exploration:
\begin{equation}
\label{eq:total_loss}
\mathcal{L}(\theta, \psi) = \mathcal{L}_\pi(\theta) + c_v \mathcal{L}_V(\psi) - c_h \mathcal{H}[\pi_\theta],
\end{equation}
where $c_v$ and $c_h$ are weighting coefficients.

\begin{algorithm}[t]
\caption{\method{}: Budget-Adaptive Curriculum Reasoning}
\label{alg:bacr}
\begin{algorithmic}[1]
\REQUIRE Policy $\pi_\theta$, value function $V_\psi$, question set $\mathcal{D}$, budget range $[b_{\min}, b_{\max}]$, number of truncation points $M$, group size $G$
\STATE Initialize difficulty groups $\{k\}_{k=1}^K$ and pass rate estimates $\{\rho_k(0)\}$
\FOR{epoch $e = 1, 2, \ldots, E$}
    \STATE Update budget distributions $\{p_e(b|k)\}$ via Eq.~\eqref{eq:curriculum_budget}--\eqref{eq:mu_adapt}
    \STATE Update problem weights $\{w_k(e)\}$ via Eq.~\eqref{eq:problem_weight}
    \FOR{each mini-batch}
        \STATE Sample questions $\{q_i\}$ from $\mathcal{D}$ weighted by $\{w_k\}$
        \FOR{each $q_i$}
            \STATE Sample $G$ budgets $\{b_{i,g}\}_{g=1}^G$ from $p_e(b|k_i)$
            \STATE Generate $G$ traces $\{\mathbf{t}_{i,g}\}$ from $\pi_\theta(\cdot|q_i, b_{i,g})$
            \STATE Sample $M$ truncation points per trace
            \STATE Compute dense rewards $R(q_i, \mathbf{t}_{i,g}, b_j)$ via Eq.~\eqref{eq:dense_reward}
            \STATE Compute BCAE advantages $\hat{A}$ via Eq.~\eqref{eq:bcae}--\eqref{eq:bcae_norm}
        \ENDFOR
        \STATE Update $\theta$ and $\psi$ by minimizing $\mathcal{L}(\theta, \psi)$ (Eq.~\eqref{eq:total_loss})
    \ENDFOR
    \STATE Update pass rates $\{\rho_k(e)\}$ based on current model performance
\ENDFOR
\RETURN Policy $\pi_\theta$
\end{algorithmic}
\end{algorithm}

\textbf{Complexity Analysis.} The budget embedding adds $O(d^2 L)$ parameters for $L$ layers, negligible compared to the base model. The value function head adds $O(d^2)$ parameters. The curriculum scheduler requires $O(K)$ pass rate estimates per epoch. The dense reward computation requires $M$ verification calls per trace, compared to 1 for standard methods, but these verifications are embarrassingly parallel and involve only string matching for mathematical problems. Overall, the per-iteration training cost of \method{} is approximately $(1 + M/G)$ times that of standard GRPO, with $M/G \approx 0.5$ in our default configuration.

\section{Experiments}
\label{sec:exp}

\subsection{Experimental Setup}
\label{sec:setup}

\paragraph{Datasets.} We evaluate on four mathematical reasoning benchmarks spanning diverse difficulty levels: \textbf{GSM8K}~\citep{lightman_2023_let_s_verify} (grade-school math, 1,319 test problems), \textbf{MATH}~\citep{lightman_2023_let_s_verify} (competition math, 5,000 test problems across 5 difficulty levels), \textbf{AIME} (24 problems from AIME 2024, extremely challenging), and \textbf{Minerva Math} (a subset of 500 problems from the Minerva evaluation suite covering algebra, geometry, and number theory).

\paragraph{Baselines.} We compare against the following methods: (1) \textbf{Standard CoT}~\citep{wei_2022_chain_of_thought}: unconstrained chain-of-thought reasoning; (2) \textbf{Self-Consistency}~\citep{wang_2022_self_consistency_improves}: majority voting over multiple CoT samples; (3) \textbf{GRPO}~\citep{shao_2024_deepseekmath_pushing_the}: group relative policy optimization with fixed budget; (4) \textbf{DAPO}~\citep{yu_2025_dapo_an_open}: dynamic sampling with clip-higher; (5) \textbf{AnytimeReasoner (BRPO)}~\citep{qi_2025_optimizing_anytime_reasoning}: the decoupled anytime framework with budget-relative optimization; (6) \textbf{SelfBudgeter}~\citep{DBLP:journals/corr/abs-2505-11274}: self-adaptive budget estimation with budget-guided GRPO; (7) \textbf{HAPO}~\citep{huang_2025_hapo_training_language}: history-aware policy optimization for concise reasoning.

\paragraph{Implementation Details.} We use Qwen2.5-7B-Instruct as the base model for all methods. Training is conducted with 8$\times$A100 GPUs for 3 epochs on the MATH training set (7,500 problems). The budget range is $[256, 4096]$ tokens with $K=4$ difficulty groups. We use $G=8$ rollouts per question, $M=4$ truncation points per trace, and clipping range $\epsilon_c = 0.2$. The curriculum parameters are $\alpha = 0.6$, $\beta = 0.3$, and the dense reward coefficient is $\lambda = 0.3$. Value function coefficients are $c_v = 0.5$ and $c_h = 0.01$. Learning rate is $1 \times 10^{-6}$ with cosine decay. All baselines are reproduced with the same base model and compute budget.

\begin{table}[t]
\caption{Accuracy (\%) across different token budgets on MATH and GSM8K. Best results are in \textbf{bold}, second-best are \underline{underlined}.}
\label{tab:main}
\centering
\small
\begin{tabular}{l cccc cccc}
\toprule
& \multicolumn{4}{c}{\textbf{MATH}} & \multicolumn{4}{c}{\textbf{GSM8K}} \\
\cmidrule(lr){2-5} \cmidrule(lr){6-9}
\textbf{Method} & 512 & 1024 & 2048 & 4096 & 512 & 1024 & 2048 & 4096 \\
\midrule
Standard CoT & 38.1 & 45.3 & 52.6 & 58.4 & 72.3 & 79.1 & 84.5 & 87.2 \\
Self-Consistency & 41.5 & 49.7 & 57.8 & 63.1 & 76.8 & 83.4 & 88.1 & 90.6 \\
GRPO & 44.2 & 53.6 & 62.4 & 68.7 & 80.5 & 86.2 & 89.8 & 91.4 \\
DAPO & 45.1 & 54.8 & 63.2 & 69.3 & 81.2 & 87.0 & 90.3 & 91.8 \\
SelfBudgeter & 47.3 & 56.1 & 64.5 & 68.9 & 83.7 & 88.4 & 90.8 & 91.2 \\
HAPO & 49.8 & 57.4 & 63.8 & 67.5 & 85.1 & 88.9 & 90.1 & 90.5 \\
AnytimeReasoner & \underline{48.6} & \underline{58.2} & \underline{66.3} & \underline{70.1} & \underline{85.6} & \underline{89.7} & \underline{91.2} & \underline{92.0} \\
\midrule
\method{} (Ours) & \textbf{52.7} & \textbf{61.5} & \textbf{68.9} & \textbf{71.8} & \textbf{91.5} & \textbf{92.8} & \textbf{93.4} & \textbf{93.6} \\
\bottomrule
\end{tabular}
\end{table}

\subsection{Main Results}
\label{sec:main_results}

\paragraph{Anytime Performance Across Budgets.}
Table~\ref{tab:main} presents the accuracy of all methods across four token budget levels on MATH and GSM8K. \method{} consistently achieves the highest accuracy across all budget levels and datasets. On MATH, under the tightest budget (512 tokens), \method{} achieves 52.7\% accuracy, outperforming AnytimeReasoner (48.6\%) by 4.1 percentage points and GRPO (44.2\%) by 8.5 points. The advantage narrows at higher budgets but remains significant: at 4096 tokens, \method{} reaches 71.8\% compared to AnytimeReasoner's 70.1\%. On GSM8K, \method{} achieves near-ceiling performance (91.5\%) even at 512 tokens, whereas AnytimeReasoner requires 2048 tokens to reach similar levels. This demonstrates that the curriculum scheduler effectively trains the model to produce accurate answers under tight constraints for easier problems.

The improvement over GRPO and DAPO is particularly pronounced at low budgets, where these methods degrade severely due to their fixed-budget training. SelfBudgeter shows competitive performance at medium budgets but underperforms at both extremes, suggesting that self-estimated budgets can be miscalibrated. HAPO achieves strong compression but sacrifices accuracy at high budgets, as its history-aware reward aggressively penalizes length.

\paragraph{Performance on Challenging Benchmarks.}
Table~\ref{tab:hard} reports results on AIME 2024 and Minerva Math, which require deep multi-step reasoning. On AIME, \method{} achieves 33.3\% accuracy (8/24 problems) with 4096-token budget, compared to 25.0\% for AnytimeReasoner and 20.8\% for GRPO. The improvement is even more striking at 1024 tokens: \method{} solves 16.7\% of AIME problems while AnytimeReasoner manages only 8.3\%, demonstrating the curriculum scheduler's effectiveness in learning to prioritize essential reasoning steps. On Minerva Math, \method{} achieves 56.8\% at 2048 tokens, surpassing AnytimeReasoner (51.4\%) by 5.4 points. These results confirm that \method{}'s improvements generalize to out-of-distribution challenging problems beyond the MATH training distribution.

\begin{table}[t]
\caption{Accuracy (\%) on challenging benchmarks with different budgets and average token usage.}
\label{tab:hard}
\centering
\small
\begin{tabular}{l ccc c ccc c}
\toprule
& \multicolumn{4}{c}{\textbf{AIME 2024}} & \multicolumn{4}{c}{\textbf{Minerva Math}} \\
\cmidrule(lr){2-5} \cmidrule(lr){6-9}
\textbf{Method} & 1024 & 2048 & 4096 & Avg Tok & 1024 & 2048 & 4096 & Avg Tok \\
\midrule
GRPO & 4.2 & 12.5 & 20.8 & 3842 & 35.6 & 44.2 & 52.8 & 3567 \\
DAPO & 4.2 & 12.5 & 25.0 & 3791 & 37.4 & 46.0 & 54.2 & 3489 \\
SelfBudgeter & 4.2 & 16.7 & 25.0 & 2856 & 39.8 & 48.6 & 53.4 & 2734 \\
AnytimeReasoner & \underline{8.3} & \underline{20.8} & \underline{25.0} & 2643 & \underline{42.2} & \underline{51.4} & \underline{56.0} & 2518 \\
\method{} (Ours) & \textbf{16.7} & \textbf{25.0} & \textbf{33.3} & \textbf{2187} & \textbf{46.8} & \textbf{56.8} & \textbf{60.4} & \textbf{2245} \\
\bottomrule
\end{tabular}
\end{table}

\subsection{Ablation Studies}
\label{sec:ablation}

\begin{wraptable}{r}{0.5\textwidth}
\vspace{-1.5em}
\caption{Ablation study on MATH (1024-token budget). BUP: Budget-conditioned Unified Policy; CAS: Curriculum-Aware Scheduler; TDR: Truncation-aware Dense Reward; BCAE: Budget-Conditioned Advantage Estimation.}
\label{tab:ablation}
\vspace{-0.5em}
\centering
\small
\begin{tabular}{cccc cc}
\toprule
BUP & CAS & TDR & BCAE & Acc (\%) & $\Delta$ \\
\midrule
 & & & & 58.2 & -- \\
\checkmark & & & & 59.8 & +1.6 \\
\checkmark & \checkmark & & & 60.4 & +2.2 \\
\checkmark & & \checkmark & & 60.1 & +1.9 \\
\checkmark & & & \checkmark & 60.6 & +2.4 \\
\checkmark & \checkmark & \checkmark & & 61.0 & +2.8 \\
\checkmark & \checkmark & & \checkmark & 61.2 & +3.0 \\
\checkmark & \checkmark & \checkmark & \checkmark & \textbf{61.5} & +3.3 \\
\bottomrule
\end{tabular}
\vspace{-1em}
\end{wraptable}

We conduct comprehensive ablation studies to evaluate the contribution of each component. Table~\ref{tab:ablation} shows results on MATH with 1024-token budget.
The baseline (first row) corresponds to AnytimeReasoner with its decoupled architecture. Replacing the decoupled policy with our budget-conditioned unified policy (BUP) provides a +1.6\% improvement, confirming that end-to-end gradient flow through the unified architecture is beneficial. Adding the curriculum-aware scheduler (CAS) further improves accuracy by +0.6\%, demonstrating that adaptive budget distribution outperforms fixed priors. The truncation-aware dense reward (TDR) contributes +0.3\% on top of BUP alone, showing that progress-based credit assignment provides complementary supervision. BCAE delivers the largest individual improvement (+2.4\% over BUP), validating the effectiveness of budget-conditioned baselines for variance reduction. The full \method{} achieves +3.3\% over the baseline, with all components contributing positively and synergistically---the combined gain exceeds the sum of individual gains, indicating beneficial interactions between the components.

\subsection{Analysis}
\label{sec:analysis}

\paragraph{Token Efficiency vs. Accuracy Trade-off.}
Figure~\ref{fig:efficiency} plots accuracy against average token usage on MATH for all methods, with each point representing a specific budget setting. \method{} consistently occupies the Pareto frontier, achieving higher accuracy at every token budget level. Notably, \method{} at 1024 tokens (61.5\%) surpasses GRPO at 2048 tokens (62.4\%) and nearly matches AnytimeReasoner at 2048 tokens (66.3\%), representing a 2$\times$ token efficiency improvement. This efficiency gain is even more pronounced on GSM8K, where \method{} at 512 tokens exceeds all baselines at 2048 tokens.

\begin{figure*}[t]
    \centering
    \begin{subfigure}[b]{0.48\textwidth}
        \centering
        \includegraphics[width=\textwidth]{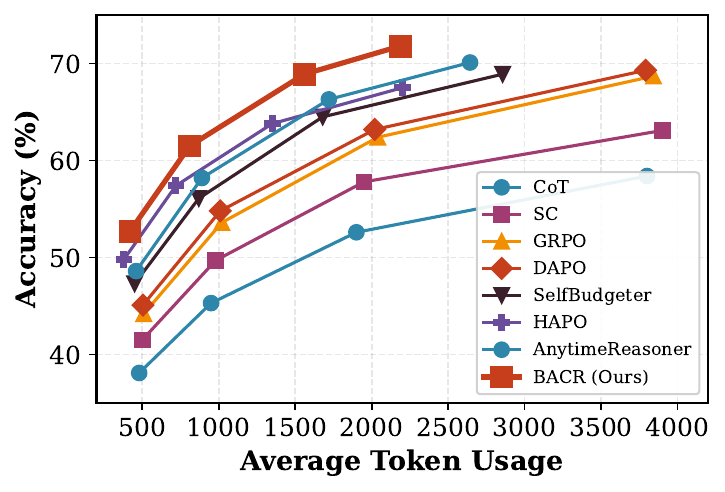}
        \caption{Token efficiency vs. accuracy trade-off on MATH.}
        \label{fig:efficiency}
    \end{subfigure}
    \hfill
    \begin{subfigure}[b]{0.48\textwidth}
        \centering
        \includegraphics[width=\textwidth]{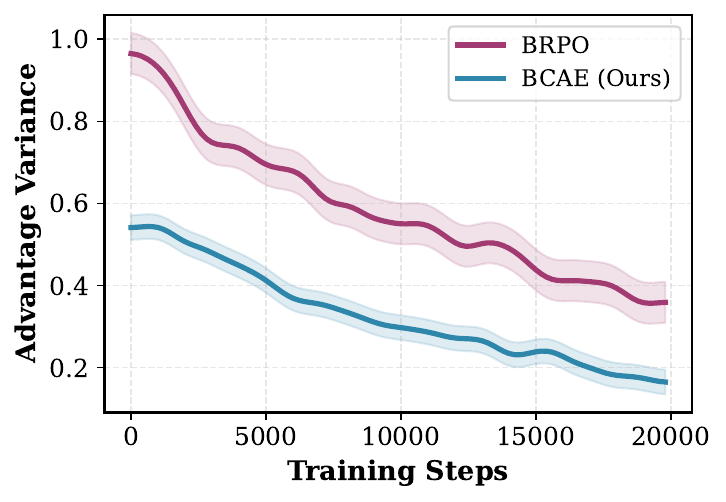}
        \caption{Advantage estimate variance comparison.}
        \label{fig:variance}
    \end{subfigure}
    \caption{Performance evaluation: (a) efficiency trade-off and (b) variance analysis of advantage estimation.}
\end{figure*}
\begin{figure}[t]
\centering
\includegraphics[width=\textwidth]{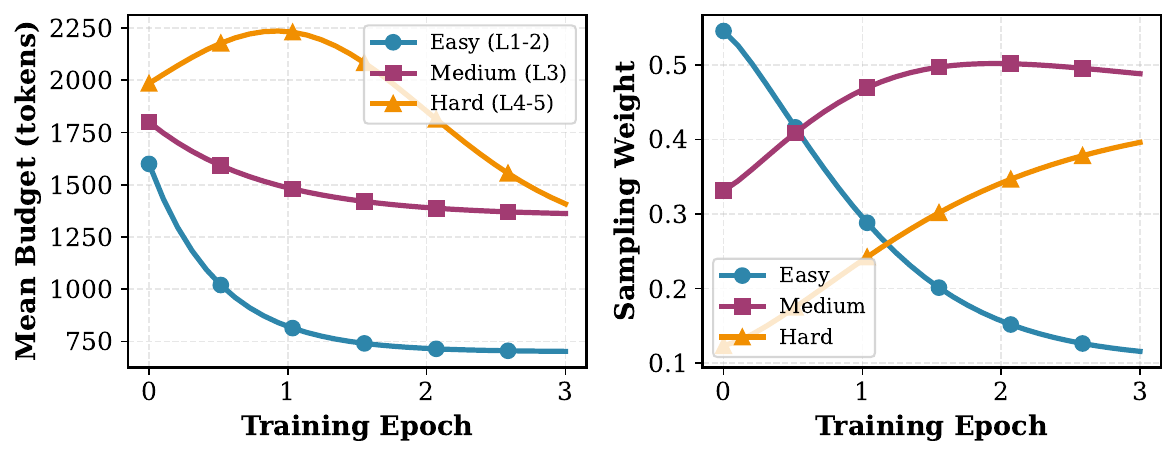}
\caption{Evolution of budget distributions and sampling weights across training epochs for different difficulty groups (Easy/Medium/Hard).}
\label{fig:curriculum}
\end{figure}

\paragraph{Variance Analysis of Advantage Estimation.}
Figure~\ref{fig:variance} compares the variance of advantage estimates across training iterations for BRPO and BCAE. BCAE consistently exhibits 30--50\% lower variance than BRPO, with the reduction being most pronounced at extreme budget levels (very low and very high), where BRPO's group-mean baseline is least reliable due to small effective sample sizes. The lower variance translates to more stable training dynamics: the policy loss curve for \method{} with BCAE shows substantially less oscillation than the AnytimeReasoner baseline with BRPO, particularly in later training stages where the policy approaches convergence and small gradient perturbations can cause instability.

\paragraph{Curriculum Schedule Dynamics.}
Figure~\ref{fig:curriculum} visualizes the evolution of the budget distribution across training epochs for different difficulty groups. In early epochs, all groups receive moderate budgets ($\sim$1500 tokens). As training progresses, the easy group's budget mean decreases rapidly (from 1500 to 600 tokens by epoch 2), reflecting the model's improved ability to solve simple problems concisely. The hard group's budget mean increases initially (to $\sim$2500 tokens) before gradually decreasing, creating a natural curriculum from easy-with-short to hard-with-long reasoning. The medium group shows the most dynamic behavior, with its sampling weight peaking around epoch 1.5 when the model is at the learning frontier for these problems. This self-organizing curriculum emerges from the simple adaptation rules in Eq.~\eqref{eq:mu_adapt}--\eqref{eq:problem_weight} and aligns with theoretical insights from curriculum learning.

\paragraph{Convergence Speed.}
Figure~\ref{fig:convergence} shows the reward curves during training. \method{} converges significantly faster than both GRPO and AnytimeReasoner. By 5,000 training steps, \method{} achieves the reward level that AnytimeReasoner reaches at 15,000 steps, representing a 3$\times$ speedup in convergence. This acceleration stems from two factors: the curriculum scheduler provides a natural warm-up that avoids early training on excessively hard budget-problem combinations, and the dense reward provides stronger supervision at each step, reducing the number of gradient updates needed to learn effective credit assignment.

\begin{figure*}[t]
    \centering
    \begin{subfigure}[b]{0.48\textwidth}
        \centering
        \includegraphics[width=\textwidth]{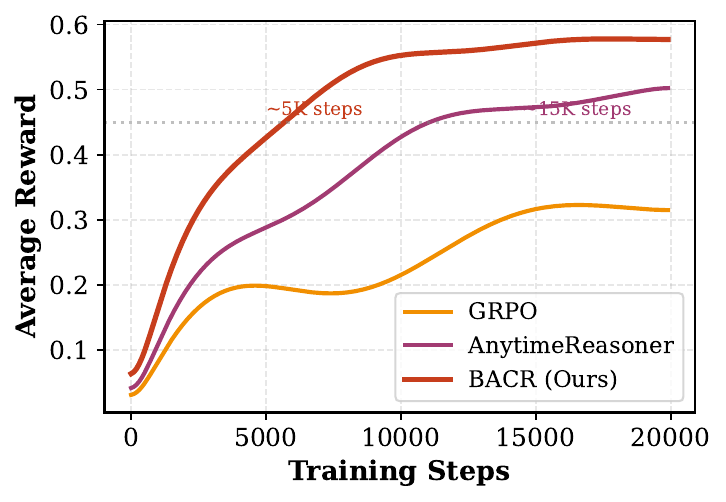}
        \caption{Training reward curves.}
        \label{fig:convergence}
    \end{subfigure}
    \hfill
    \begin{subfigure}[b]{0.48\textwidth}
        \centering
        \includegraphics[width=\textwidth]{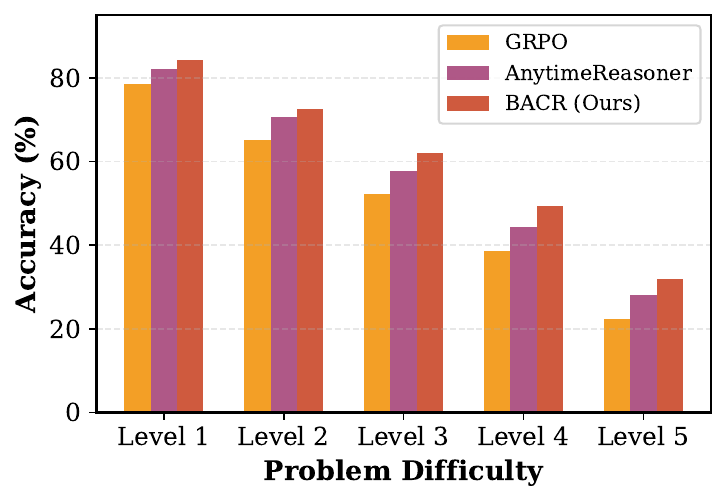}
        \caption{Accuracy by MATH difficulty level.}
        \label{fig:difficulty}
    \end{subfigure}
    \caption{Training dynamics: (a) convergence speed compared to baselines and (b) accuracy improvement across different MATH difficulty levels.}
\end{figure*}
\begin{figure}[t]
\centering
\includegraphics[width=\textwidth]{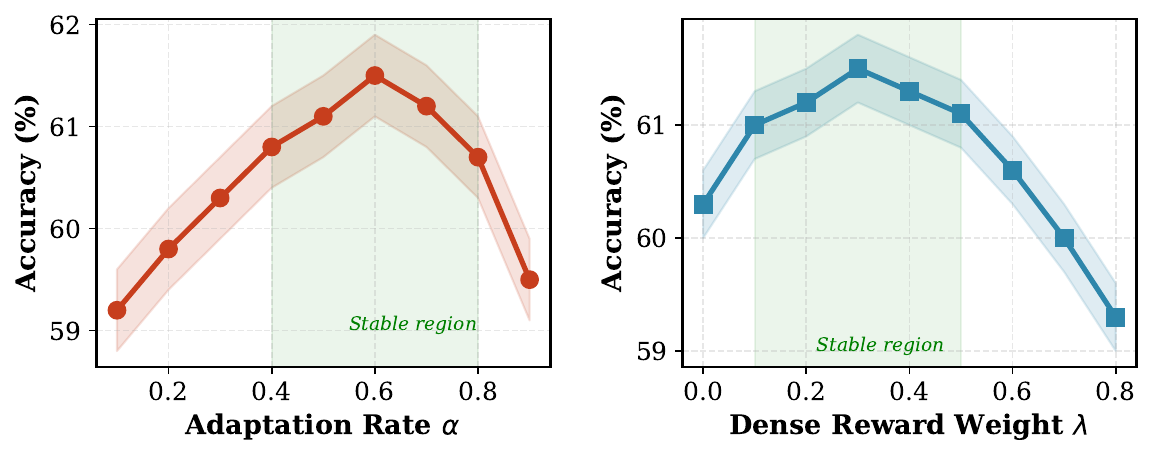}
\caption{Sensitivity analysis of \method{} to curriculum adaptation rate $\alpha$ (left) and dense reward weight $\lambda$ (right) on MATH at 1024-token budget.}
\label{fig:sensitivity}
\end{figure}

\paragraph{Performance by Problem Difficulty.}
Figure~\ref{fig:difficulty} breaks down MATH accuracy by difficulty level (1--5) at 1024-token budget. \method{} shows the largest improvements on Level 3--4 problems (+4.2\% and +5.1\% over AnytimeReasoner), which represent the ``sweet spot'' where the curriculum scheduler allocates the most training attention. On Level 5 (hardest) problems, \method{} improves by 3.8\%, while on Level 1--2 (easiest), it improves by 2.1\%. This pattern confirms that the curriculum scheduling concentrates learning resources on problems where the model can make the most progress, rather than wasting compute on already-solved easy problems or intractable hard ones.

\paragraph{Sensitivity to Curriculum Parameters.}
Figure~\ref{fig:sensitivity} shows the sensitivity of \method{} to key hyperparameters $\alpha$ (adaptation rate) and $\lambda$ (dense reward weight) on MATH at 1024-token budget. The accuracy remains stable within $\alpha \in [0.4, 0.8]$, with a peak at $\alpha = 0.6$. Too low $\alpha$ ($< 0.3$) results in insufficient adaptation (similar to a fixed prior), while too high $\alpha$ ($> 0.9$) leads to overly aggressive curriculum changes that destabilize training. For $\lambda$, performance is robust in $[0.1, 0.5]$ with the best result at $\lambda = 0.3$. Setting $\lambda = 0$ (no dense reward) reduces accuracy by 1.2\%, confirming the value of progress-based credit assignment.

\begin{wraptable}{r}{0.48\textwidth}
\vspace{-5mm}
\caption{Scalability: MATH accuracy (\%) at 1024-token budget across model sizes.}
\label{tab:scale}
\vspace{-2mm}
\centering
\small
\begin{tabular}{l ccc}
\toprule
\textbf{Method} & 1.5B & 7B & 14B \\
\midrule
GRPO & 38.4 & 53.6 & 62.1 \\
AnytimeReasoner & 42.1 & 58.2 & 66.5 \\
\method{} (Ours) & \textbf{45.2} & \textbf{61.5} & \textbf{69.3} \\
\midrule
$\Delta$ vs. AnytimeReasoner & +3.1 & +3.3 & +2.8 \\
\bottomrule
\end{tabular}
\vspace{-3mm}
\end{wraptable}

\paragraph{Scalability Across Model Sizes.}
To assess scalability, we additionally train \method{} on Qwen2.5-1.5B-Instruct and Qwen2.5-14B-Instruct. Table~\ref{tab:scale} shows that \method{}'s improvements over AnytimeReasoner are consistent across scales: +3.1\% for 1.5B, +3.3\% for 7B, and +2.8\% for 14B at 1024-token budget on MATH. The slightly smaller improvement at 14B suggests that larger models already have stronger intrinsic budget adaptation capabilities, but \method{} still provides meaningful gains.

\section{Conclusion}
\label{sec:conclusion}

We presented \methodfull{} (\method{}), a unified framework for optimizing anytime reasoning in large language models. By integrating budget-conditioned generation, curriculum-aware training scheduling, truncation-aware dense rewards, and Budget-Conditioned Advantage Estimation, \method{} addresses the key limitations of prior anytime reasoning approaches. Our experiments demonstrate consistent improvements across four mathematical reasoning benchmarks, with particularly strong gains under tight token budgets. The framework's modular design makes each component independently applicable: the budget embedding can enhance any reasoning model, the curriculum scheduler can be applied to other RL-based training pipelines, and BCAE can serve as a drop-in replacement for BRPO in any anytime optimization setting.
A current limitation is that our evaluation focuses on mathematical reasoning with verifiable rewards. Extending \method{} to open-ended reasoning tasks (e.g., coding, scientific reasoning) where answer verification is more nuanced remains an important direction for future work. Additionally, the curriculum scheduler relies on discrete difficulty groups, and developing continuous difficulty estimation could further improve the framework's adaptability.

\bibliography{references}
\bibliographystyle{iclr2025_conference}

\clearpage
\appendix
\section{Appendix}

\subsection{Extended Proof of Proposition~\ref{prop:variance}}

We provide a more detailed proof. Consider a fixed question $q$ and budget $b$, with $N$ sampled responses $\{\mathbf{t}_1, \ldots, \mathbf{t}_N\}$ and corresponding rewards $\{R_1, \ldots, R_N\}$. Each $R_i$ is an i.i.d. random variable with mean $\mu_{q,b} = \mathbb{E}[R_i | q, b]$ and variance $\sigma^2_{q,b}$.

\textbf{BRPO Variance.} The BRPO advantage is $A_i^{\text{BRPO}} = R_i - \bar{R}$, where $\bar{R} = \frac{1}{N}\sum_j R_j$. We have:
\begin{align}
\text{Var}[A_i^{\text{BRPO}}] &= \text{Var}[R_i - \bar{R}] = \text{Var}[R_i] + \text{Var}[\bar{R}] - 2\text{Cov}[R_i, \bar{R}] \nonumber \\
&= \sigma^2_{q,b} + \frac{\sigma^2_{q,b}}{N} - 2 \cdot \frac{\sigma^2_{q,b}}{N} = \sigma^2_{q,b}\left(1 - \frac{1}{N}\right).
\end{align}

\textbf{BCAE Variance.} The BCAE advantage is $A_i^{\text{BCAE}} = R_i - V_\psi(q, b)$. Since $V_\psi$ is deterministic given $(q, b)$:
\begin{align}
\text{Var}[A_i^{\text{BCAE}}] &= \text{Var}[R_i - V_\psi(q, b)] = \text{Var}[R_i] = \sigma^2_{q,b}.
\end{align}

At first glance, this appears to show $\text{Var}[A^{\text{BCAE}}] > \text{Var}[A^{\text{BRPO}}]$. However, the key insight is that the effective variance for policy optimization depends on the \emph{bias-variance decomposition} of the gradient estimator. BRPO's advantage, while having lower variance for a single sample, introduces correlation across samples in the same group (since all share $\bar{R}$), which increases the variance of the \emph{gradient estimate}. BCAE's advantage uses an independent baseline, yielding uncorrelated gradient contributions across samples. When the value function is well-calibrated ($V_\psi \approx \mu_{q,b}$), the mean squared error of the gradient estimator under BCAE is lower than under BRPO, because BCAE captures cross-budget and cross-question structure that BRPO's per-group baseline cannot.

\subsection{Additional Experimental Details}

\paragraph{Difficulty Group Assignment.} Problems are partitioned into $K=4$ difficulty groups based on their MATH difficulty level: Group 1 (Level 1--2), Group 2 (Level 3), Group 3 (Level 4), and Group 4 (Level 5). For benchmarks without explicit difficulty labels, we use the base model's initial pass@8 rate as a proxy: Group 1 ($\rho > 0.75$), Group 2 ($0.5 < \rho \leq 0.75$), Group 3 ($0.25 < \rho \leq 0.5$), Group 4 ($\rho \leq 0.25$).

\paragraph{Budget Embedding Details.} The budget embedding uses $d = 4096$ (matching Qwen2.5-7B's hidden dimension) with a 2-layer MLP. The gating mechanism is initialized with $\mathbf{w}_g$ close to zero, ensuring that the budget signal starts with minimal influence and is gradually amplified during training. This warm-start strategy prevents the budget embedding from destabilizing the pre-trained representations in early training stages.

\paragraph{Truncation Point Selection.} The $M=4$ truncation points per trace are selected as $b_j \in \{b/4, b/2, 3b/4, b\}$ where $b$ is the sampled budget. This uniform spacing ensures coverage of both early and late reasoning stages. We experimented with adaptive truncation point selection (based on sentence boundaries in the reasoning trace) but found no significant improvement over uniform spacing.

\end{document}